\documentclass[10pt,twocolumn,letterpaper]{article}

\usepackage{wacv}
\usepackage{times}
\usepackage{epsfig}
\usepackage{graphicx}
\usepackage{amsmath}
\usepackage{amssymb}
\usepackage[accsupp]{axessibility}  

\usepackage[T1]{fontenc}

\usepackage{multirow}

\usepackage{amsthm,amsfonts,dsfont}
\usepackage{verbatim}
\usepackage{caption}
\usepackage{algorithmic}
\usepackage{subcaption}
\usepackage{bm}
\usepackage{paralist}

\renewcommand{\algorithmiccomment}[1]{\bgroup\hfill//~#1\egroup}
\newcommand{\INDSTATE}[1][1]{\STATE\hspace{#1\algorithmicindent}}

\newtheorem{theorem}{Theorem}[section]

\def\R{\mathbb{R}}

\def\1{\mathds{1}}

\def\diag{\mathrm{diag}}

\def\cov{V}

\def\erf{\mathrm{erf}}
\def\relu{\mathrm{ReLU}}

\def\conv{\mathbf{M}}

\def\our{MisConv}

\def\wacvPaperID{704} 

\wacvfinalcopy 

\ifwacvfinal
\fi

\ifwacvfinal
\usepackage[breaklinks=true,bookmarks=false]{hyperref}
\else
\usepackage[pagebackref=true,breaklinks=true,colorlinks,bookmarks=false]{hyperref}
\fi

\begin{document}

\title{\our{}: Convolutional Neural Networks for Missing Data}

\author{Marcin Przewięźlikowski\qquad
Marek \'Smieja\qquad
\L{}ukasz Struski\qquad
Jacek Tabor\\[0.5em]

Faculty of Mathematics and Computer Science, Jagiellonian University\\
6~\L{}ojasiewicza Street, 30-348 Krak\'ow, Poland\\[0.5em]

{\tt\small marcin.przewiezlikowski@student.uj.edu.pl}\\
{\tt\small \{marek.smieja, lukasz.struski, jacek.tabor\}@uj.edu.pl}
}

\maketitle
\ifwacvfinal
\thispagestyle{empty}
\fi
\begin{abstract}
   Processing of missing data by modern neural networks, such as CNNs, remains a  fundamental, yet unsolved challenge, which naturally arises in many practical applications, like image inpainting or autonomous vehicles and robots. While imputation-based techniques are still one of the most popular solutions, they frequently introduce unreliable information to the data and do not take into account the uncertainty of estimation, which may be destructive for a machine learning model. In this paper, we present \our{}, a general mechanism, for adapting various CNN architectures to process incomplete images. By modeling the distribution of missing values by the Mixture of Factor Analyzers, we cover the spectrum of possible replacements and find an analytical formula for the expected value of convolution operator applied to the incomplete image. The whole framework is realized by matrix operations, which makes \our{} extremely efficient in practice. Experiments performed on various image processing tasks demonstrate that \our{} achieves superior or comparable performance to the state-of-the-art methods\footnote{Code available at \url{https://github.com/mprzewie/dmfa_inpainting}.}.
\end{abstract}

\section{Introduction}

\begin{figure*}
    \centering
    \includegraphics[width=\textwidth]{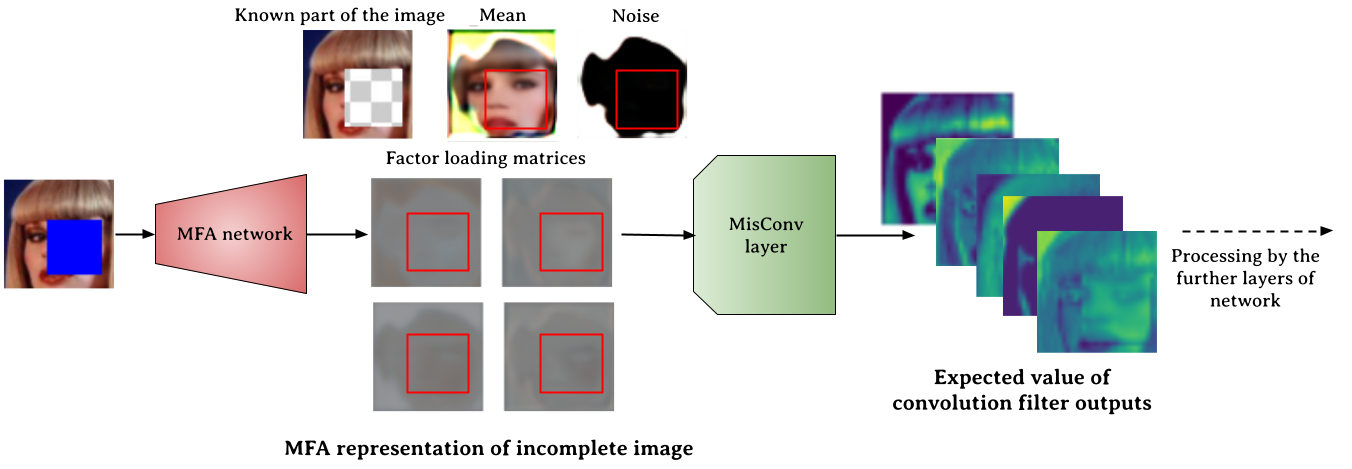}
    \caption{\our{} redefines the first convolutional layer of existing CNN to process incomplete images. First, it estimates a distribution of possible replacements for an incomplete image using a Mixture of Factor Analyzers. Next, the convolutional layer computes the expected value of the convolutional operator taken over the MFA distribution. This mechanism takes the uncertainty of imputation into account and can be implemented efficiently using matrix operations.
    }
    \label{fig:scheme}
\end{figure*}

Convolutional neural networks (CNNs) present state-of-the-art performance on various image processing tasks, such as classification, segmentation, object detection \cite{krizhevsky2012imagenet, karpathy2014large, ciresan2011flexible}. Nevertheless, standard CNNs cannot directly process images, which contain missing regions. Since missing data are very common in medical diagnosis, self-driving vehicles or robots, and other computer vision applications, generalizing CNNs to this case is of great importance.

The typical strategy for using neural networks with incomplete images relies on replacing absent pixels with the most probable candidates in the initial stage \cite{cheng2019novel, li2019misgan}. While imputation-based techniques can be combined with various machine learning models, they do not take into account the uncertainty of estimation. In particular, more than one replacement usually conforms to a given image, but only a single one is selected in the imputation stage\footnote{In the case of multiple imputations, a selected number of candidates are considered.}. The use of such a hard decision may be destructive because important information about data distribution is lost during this process.

To overcome the aforementioned problems, many classical machine learning models, such as logistic regression or SVM, were adapted to learn from missing data directly \cite{chechik2008max, dekel2010learning, globerson2006nightmare}. They either ignore missing values or integrate the probability distribution of missing data to make a final prediction. Although these approaches do not substitute missing values with a single point estimate, the works on adapting neural networks, including CNNs, to the case of incomplete images are limited \cite{smieja2018processing, liu2018partialinpainting}.

In this paper, we propose a general mechanism, \our{}, for adapting CNN architectures to incomplete images. \our{} does not rely on a single imputation, but takes the uncertainty contained in missing pixels into account. Our idea is to model the probability distribution of missing values and calculate the expected value of the convolution operator taken over all possible replacements -- see Figure \ref{fig:scheme} for the illustration. Making use of the modified Mixture of Factor Analyzers (MFA) \cite{richardson2018gans, przewiezlikowskiestimatingICONIP}, we model the distribution of missing pixels with sufficient accuracy as well as we are able to find an analytical formula for the expected value of convolution. Although related approaches have been considered before for shallow models, such as logistic regression \cite{williams2005incomplete} and kernel methods \cite{smieja2019generalized}, or fully connected neural networks \cite{smieja2018processing}, there are no works on applying similar ideas to CNNs.

\our{} can be applied to various CNN architectures and requires only a modification of the first convolutional layer. The whole framework is realized by matrix operations and classical convolutions. In consequence, the complexity of \our{} is comparable to typical CNNs operating on complete images. The proposed approach is experimentally verified on various image processing tasks, including image classification, reconstruction, and generation. We demonstrate that \our{} compares favorably with state-of-the-art techniques used for processing incomplete images, which confirms its practical usefulness.

Our contributions can be summarized as follows:
\begin{compactitem}
    \item We introduce a general mechanism for adapting CNNs to the case of incomplete data, which takes into account the uncertainty contained in missing values.  
    \item We show that \our{} can be easily implemented using typical matrix operations, which makes it extremely efficient in practice.
    \item Experimental results demonstrate that \our{} can be successfully applied to various tasks and CNN architectures, achieving state-of-the-art performance.
\end{compactitem}

\section{Related work}

\paragraph{Imputation-based techniques} One of the most common approaches for applying machine learning models to the case of missing data relies on completing absent attributes and next using a given model on complete inputs. Missing attributes may be replaced using simple statistical techniques, such as mean value or k-nn imputation \cite{cheng2019novel}, or by applying more advanced machine learning methods \cite{sharpe1995dealing, sovilj2016extreme}. Although iterative regression approaches \cite{buuren2011micemice, azur2011multiple} are frequently used for small data, they are computationally too demanding for large high dimensional data such as images. The authors of \cite{gondara2017multiple} used denoising autoencoders for multiple imputations. It was shown that zero imputation can also obtain competitive results if only the rate of missingness is carefully taken into account. There are also many works that apply deep learning methods for image inpainting \cite{li2017generative,pathak2016context, yu2018generative, smieja2020iterative}, but they require complete data for training, which is in contradiction with many real data sets (such as medical ones).



\paragraph{Adaptation of shallow models} Replacing or deleting missing values is not always necessary when the task is to perform a prediction, e.g. classification, regression, etc. \cite{ipsen2020deal}. Given a rough estimation of missing data density by GMM (Gaussian mixture model) \cite{ghahramani1994supervised, tresp1994training}, one can adapt logistic regression \cite{williams2005incomplete} or kernel methods \cite{smola2005kernel, williams2005analytical, smieja2019generalized, mesquita2019gaussian} to operate on incomplete data directly. The authors of \cite{dick2008learning, liao2007quadratically} estimated the parameters of the probability model and the classifier jointly. Decision function can also be learned based on the visible inputs alone \cite{dekel2010learning, globerson2006nightmare}, see \cite{chechik2008max, xia2017adjusted} for SVM and random forest. The authors of \cite{hazan2015classification} designed an algorithm for kernel classification under the low-rank assumption, while Goldberg et. al. \cite{goldberg2010transduction} used matrix completion strategy to solve the missing data problem. Pelckmans et. al. \cite{pelckmans2005handling} modeled the expected risk under the uncertainty of the predicted outputs. Khosravi et al. embedded logistic regression classifier by computing the expected prediction on a given feature distribution \cite{khosravi2019expect}.

\paragraph{Adaptation of deep learning models} 
Partial convolution is a general technique that redefines convolutional layers in CNNs and deals with arbitrary missing patterns in images \cite{liu2018partialinpainting}. Danel et al. generalized GCNs (graph convolutional networks) to operate on incomplete images and show that this method works well on small images \cite{danel2020processing}.  The authors of \cite{ghorbani2018embedding}, proposed a trainable embedding of incomplete inputs, which can be used as an input layer to a given neural network. Mesquita et al. adapted neural networks with random weights to the case of incomplete data by estimating the weights of the output layer as a function of the uncertainty of the missing data \cite{mesquita2019artificial}. The paper \cite{bengio1996recurrent} used recurrent neural networks with feedback into the input units, which fills absent attributes for the sole purpose of minimizing a learning criterion. By applying the rough set theory, the authors of \cite{nowicki2016novel} presented a feedforward neural network that gives an imprecise answer as the result of input data imperfection. Goodfellow et. al. \cite{goodfellow2013multi} introduced the multi-prediction deep Boltzmann machine, which is capable of solving different inference problems, including classification with missing inputs. 

Many recent works focus on imputation strategies. Yoon et al. adapted conditional GANs to fill missing values \cite{yoon2018gain}. In this approach, discriminator is trained to classify which components were observed and which have been imputed. Mattei and Frellsen introduced MIWAE technique, which allows for training autoencoders on missing data by approximating the maximum likelihood of only the observed portion of the data \cite{mattei2019miwae}. The authors of \cite{li2019acflow} proposed a flow-based model (ACFlow), which gives an explicit formula for a conditional density function of missing values.


\'Smieja et al. defined a generalized neuron's response, which computes the expected activation for missing data points, and applied this mechanism for fully connected networks \cite{smieja2018processing}. However, this work ignores the case of CNNs and does not show how to compute the expected activation in this situation. Moreover, modeling high dimensional data, such as images, using Gaussian mixture models (GMMs) requires special attention. On the one hand, the number of parameters of Gaussian distribution with full covariance matrix grows quadratically with data dimension, while on the other hand, restricting to diagonal covariance is insufficient for modeling images \cite{richardson2018gans}. In this paper, we show how to efficiently apply a similar probabilistic approach for CNNs and image data.

\section{Convolutional layer for missing data}

\paragraph{Model overview} We consider the problem of processing missing data by convolutional neural networks (CNNs). Typical CNNs can only process complete images, in which values of all pixels are known. For incomplete images, the convolutional operation is not defined and there appears a question of how to harness the power of CNNs in this case.

We tackle this problem using a probabilistic approach, where missing data is represented by probability distributions. Such a representation takes into account the uncertainty of missing data estimation, which is more reliable than using a single or multiple imputation. As a probabilistic model, we use a Gaussian density represented in the form of the Mixture of Factor Analyzers (MFA), which was successfully used to model image distributions~\cite{richardson2018gans}. Moreover, its Gaussian form allows for performing an analytical calculation on this random vector.

To process the probabilistic representation of missing data, we adapt the convolution operator. Here, we introduce a convolution that is a generalization of typical convolution, which computes the mean activation (expected value) taken over all possible replacements. Formally, if $\mathbf{Z}$ is a random vector representing incomplete image $\mathbf{x}$, $\conv$ is linear convolution and $f$ is the activity function, our convolutional layer computes:
$$
f(\conv \mathbf{x}) := \mathbb{E}[f(\conv \mathbf{Z})].
$$
Since the expected value transforms random variables to numeric values, our convolution has to be applied only in the first layer. As a benefit, the rest of the architecture can be left unchanged, which is very important in practice, because typical neural networks can be directly reused in our framework. 

\paragraph{Probabilistic representation of incomplete images} A missing data point is denoted by $\mathbf{x} = (\mathbf{x}_o, \mathbf{x}_m) \in \R^n$, where $\mathbf{x}_o \in \R^{d}$ represents pixels with known values, while $\mathbf{x}_m \in \R^{n-d}$ denotes absent pixels. The set of indices (pixels) with missing values at sample $\mathbf{x}$ is denoted $\mathcal{J} \subset \{1,\ldots,n\}$. Formally $\mathcal{J} = \mathcal{J}_\mathbf{x}$, which means that the missing pattern can be different for every instance $\mathbf{x}$, but we drop the subscript $\mathbf{x}$ to simplify notation. 

To account for uncertainty contained in missing pixels, we employ a conditional probability distribution $p_{\mathbf{x}_m|\mathbf{x}_o}$. Making use of conditional density, we can cover a wider spectrum of possible replacements than using point estimates. 
While the conditional density $p_{\mathbf{x}_m|\mathbf{x}_o}$ is defined on $(n-d)$-th dimensional space, we make its natural extension to the whole $\R^n$ space, by putting 
$$
P_{\mathbf{x}_m|\mathbf{x}_o}(\mathbf{t}) = \left\{
\begin{array}{ll}
p_{\mathbf{x}_m|\mathbf{x}_o}(\mathbf{t}_{\mathcal{J}'}), & \text{ if } \mathbf{t}_{\mathcal{J}'} = \mathbf{x}_o,\\
0, & \text{ otherwise},
\end{array}
\right.
$$
where $\mathbf{t}_{\mathcal{J}'}$ denotes the restriction of $\mathbf{t} \in \R^n$ to the observed pixels $\mathcal{J}' = \{1,\ldots,n\} \setminus \mathcal{J}$.

As a parametric model of conditional density $p_{\mathbf{x}_m|\mathbf{x}_o}$, we select the Mixture of Factor Analyzers (MFA), a type of Gaussian mixture model (GMM) \cite{ghahramani1996algorithm}. Let us recall that a single component of MFA, called Factor Analyzer (FA), has a Gaussian distribution with a covariance matrix spanned on a low dimensional space, which drastically reduces the number of model parameters. In particular, the memory and complexity of estimating MFA grow linearly with data dimension (not quadratically as in the standard GMM). Moreover, the Gaussian form of MFA allows us to perform analytical calculations on the conditional density. In the case of deep generative models, such as GANs \cite{Goodfellow}, VAEs \cite{VAE} or INFs \cite{dinh2014nice}, analytical calculation is difficult or even impossible and sampling remains as the only possible strategy. Finally, it was recently shown that MFA is competitive to simple versions of GANs in estimating a density of high dimensional data such as images \cite{richardson2018gans}. Since we do not need sharp, realistic replacements, but only a rough estimation of possible values, MFA suits our needs perfectly.


Formally, a single Factor Analyzer (FA) defined in $\R^n$ is described by the mean vector $\bm{\mu} \in \R^n$ and the covariance matrix $\mathbf{\Sigma} = \mathbf{A} \mathbf{A}^T + \mathbf{D}$, where $\mathbf{A}_{n \times l}$ is a low rank factor loading matrix composed of $l$ vectors $\mathbf{a}_1,\ldots,\mathbf{a}_l \in \R^n$, such that $l \ll n$, and $\mathbf{D} = \mathbf{D}_{n \times n} = \diag(\mathbf{d})$ is a diagonal matrix representing noise defined by $\mathbf{d} \in \R^n$. Formally, FA is modeled as a random vector defined by:
\begin{equation} \label{eq:wl}
\mathbf{Z}=\bm{\mu}+\sqrt{\mathbf{d}} \odot\mathbf{X}+\sum_{j=1}^l Y_j \cdot \mathbf{a}_j,
\end{equation}
where $\mathbf{X} \sim N(0,\mathbf{I})$, $Y_j \sim N(0,1)$ are independent,  $\sqrt{\mathbf{d}}$ denotes element-wise square root of vector $\mathbf{d}$, and $\mathbf{a} \odot \mathbf{b}$ stands for element-wise multiplication of vectors $\mathbf{a}$ and $\mathbf{b}$.


\paragraph{Linear transformation of missing data}

We consider a random vector $\mathbf{Z}$ with MFA distribution $P_\mathbf{Z}$ representing a missing data point $\mathbf{x}=(\mathbf{x}_o,\mathbf{x}_m)$. The introduced convolutional layer allows us to transform a probabilistic MFA representation $\mathbf{Z}$ into a numeric output vector. Our idea is to calculate the expected value of the convolution applied to $\mathbf{Z}$. Below, we present an analytical formula for this transformation.

Let $\conv$ be a linear convolution operator (without applying nonlinear activity function), which produces a random vector $\conv \mathbf{Z}$ when applied to $\mathbf{Z}$. First, we show that if $\mathbf{Z}$ is a Factor Analyzer then $\conv \mathbf{Z}$ is also a Factor Analyzer. 
\begin{theorem}\label{thm:rv}
Let $\mathbf{Z}$ be a random vector with FA distribution defined by $\mathbf{Z}=\bm{\mu}+\sqrt{\mathbf{d}} \odot \mathbf{X}+\sum_{j=1}^l Y_j \cdot \mathbf{a}_j$. Then the random vector $\conv \mathbf{Z}$ has Factor Analyzer distribution with mean and variance given by
\begin{equation}\label{eq:conv}
\begin{array}{l}
\mathbb{E}[\conv \mathbf{Z}]  = \conv \bm{\mu}\\[0.8ex]
\mathbb{\cov}[\conv \mathbf{Z}] = \diag(\conv \mathbf{d}) + \sum_{j=1}^l (\conv \mathbf{a}_j) \cdot (\conv \mathbf{a}_j)^T
\end{array}
\end{equation}
\end{theorem}
\begin{proof}
Observe that $\conv \mathbf{Z}$ produces Gaussian distribution, because $\conv$ is a linear operator. Its parameters can be calculated as follows:
$$
\begin{array}{ll}
\mathbb{E}[\conv \mathbf{Z}] & =\conv \mathbb{E}[\mathbf{Z}]=\conv \bm{\mu},\\[0.8ex]
\mathbb{\cov}[\conv \mathbf{Z}] & =\mathbb{\cov}[\conv(\sqrt{\mathbf{d}}.\mathbf{X})]+\sum_{j=1}^l \mathbb{\cov}[\conv(Y_j \cdot \mathbf{a}_j)]\\
& =\diag(\conv \mathbf{d}) +\sum_{j=1}^l \mathbb{\cov}[Y_i \cdot \conv \mathbf{a}_j]\\[0.8ex]
&=\diag(\conv \mathbf{d}) +\sum_{j=1}^l (\conv \mathbf{a}_j) \cdot (\conv \mathbf{a}_j)^T.
\end{array}
$$
\end{proof}
If $\mathbf{Z}$ is the MFA then we need to apply the convolution operator to every component according to Theorem \ref{thm:rv} (the mixture proportions remains the same).

\paragraph{Expected activation}

The next step is the application of the activation function to every coordinate of the feature map produced by $\conv \mathbf{Z}$. Thus, we ignore the correlations between coordinates and take into account only the elements on the diagonal of the covariance matrix. The diagonal of $\conv \mathbf{Z}$ from Theorem \ref{thm:rv} is given by:
\begin{equation}\label{eq:mar}
\conv \mathbf{d}+\sum_{j=1}^l (\conv \mathbf{a}_j) \odot (\conv \mathbf{a}_j).
\end{equation}
Due to the above marginalization, it is sufficient to consider 1-dimensional MFA representing the distribution of $\conv \mathbf{Z}$ on each coordinate of the feature map. Since we deal with random vectors, we calculate the expected value of the activation function. We restrict our attention to ReLU activation because it is commonly used in CNNs. 

\begin{theorem}
Let $P = \sum_{i=1}^k p_i N(m_i, \sigma_i^2)$ be 1-dimensional Gaussian density. The expected value of ReLU applied to random variable with density $P$ equals:
\begin{multline}\label{eq:expected}
\mathbb{E}[\relu(P)] = \\
\frac{1}{2}\sum_{i=1}^k p_i \left(m_i + 
  \frac{\sigma_i}{2\sqrt{2 \pi}} \exp(-\frac{m_i^2}{2\sigma_i^2}) + m_i \cdot \erf(\frac{m_i}{\sigma_i\sqrt{2}})\right),
\end{multline}
where $\erf(z) = \frac{2}{\sqrt{\pi}} \int_0^z \exp(-t^2)dt$ is the error function.
\end{theorem}
\begin{proof}
First, observe that:
\begin{align*}
    \mathbb{E}[\relu(P)] = \int_\R \relu(x) \sum_i N(m_i,\sigma_i^2)(x) dx \\
    = \sum_i p_i \int_0^\infty x N(m_i,\sigma_i^2)(x)dx.
\end{align*}
Finally, the last integral can be calculated as follows:
\begin{multline*}
\int_0^\infty x N(m,\sigma^2)(x)dx =\\
\frac{1}{2}\left(m + \frac{\sigma}{2\sqrt{2\pi}} \exp(-\frac{m^2}{2\sigma^2}) +m \cdot \erf(\frac{m}{\sigma \sqrt{2}})\right).
\end{multline*}
\end{proof}

\paragraph{Practical realization}
\our{} allows us to efficiently adapt typical CNNs to the case of incomplete data by redefining the first hidden layer. The whole procedure is summarized in Algorithm 1.

\begin{figure*}
{\bf Algorithm 1:}
\medskip
\hrule
\medskip
\begin{algorithmic}[1] 
 \STATE \textbf{INPUT:}
  \INDSTATE $\mathbf{x} = (\mathbf{x}_o,\mathbf{x}_m)$ -- incomplete image\\
  \STATE \textbf{OUTPUT:}\\
  \INDSTATE $\relu(\conv \mathbf{x}))$ -- transformation of $\mathbf{x}$ by the 1st hidden layer\\
 \STATE \textbf{1st HIDDEN LAYER}\\
 \INDSTATE Compute a density $F_\mathbf{Z} = \sum_{i=1}^k p_i N(\bm{\mu}_i, \mathbf{A}_i^T \mathbf{A}_i + \diag(\mathbf{d}_i))$ of MFA representation $\mathbf{Z}$ of $\mathbf{x}$.
 \INDSTATE Compute a distribution $F_{\conv \mathbf{Z}} = \sum_{i=1}^k p_i N(\mathbf{m}_i,\diag(\bm{\sigma}_i))$ of $\conv \mathbf{Z}$ on every pixel where:
 \INDSTATE[2] $\mathbf{m}_i = (m_i^1,\ldots,m_i^n) = \conv \bm{\mu}_i$
 \INDSTATE[2] $\bm{\sigma}_i = (\sigma_i^1,\ldots,\sigma_i^n) = \conv \mathbf{d}+\sum_{j=1}^l (\conv \mathbf{a}_j) \odot (\conv \mathbf{a}_j)$
 \INDSTATE Compute the expected ReLU activation of $\sum_{i=1}^k p_i N(m_i^j,(\sigma_i^j)^2)$ on every pixel $j$ using \eqref{eq:expected}
\end{algorithmic}
\medskip
\hrule
\end{figure*}


As can be seen, \our{} can be implemented with the use of typical matrix operations and classical convolutions. The first layer applies a convolution operator $\conv$ to the mean vector $\bm{\mu}$ of every FA as well as to each column of matrix $\mathbf{A}$ and noise vector $\mathbf{d}$ defining the covariance matrix. Since \our{} involves only the modification in the first layer, the remaining part of the CNN architecture can be left unchanged. 
We experimentally verify that the computational overhead introduced by this modification is small and has little effect on the overall network performance, see the supplementary materials.

Incomplete image $\mathbf{x}$ can be represented by MFA in two ways. One way is to estimate MFA from an incomplete data set and compute conditional MFA for every image $\mathbf{x}$ \cite{richardson2018gans}. An alternative option relies on direct estimation of conditional MFA using density network \cite{przewiezlikowskiestimatingICONIP, bishop1994mixture}. In the latter case, we construct a neural network that takes incomplete image $\mathbf{x}$ and returns the parameters of a conditional MFA. Such an inpainting network is trained using log-likelihood loss. Experimental evaluation, presented in the following section, demonstrates that this approach leads to better results, see Section~\ref{sec:an}.

\section{Experiments}
\label{sec:experiments}

We evaluate \our{} on three machine learning problems: classification, image generation and image reconstruction\footnote{The code implementing our technique is added to the supplemental material and will be made publicly available when the review period ends.}. 

\paragraph{Experimental setting} If not stated otherwise, every model is trained and tested on missing data. To generate missing regions, we remove a square patch covering 1/4 area of the image. The location of the patch is uniformly sampled for each example. In the supplementary materials, we also consider more challenging tasks, where larger portions of images are missing.

\our{} is parametrized by MFA, which builds a probabilistic representation of incomplete images. For this purpose, we use DMFA (Deep Mixture of Factor Analyzers) \cite{przewiezlikowskiestimatingICML, przewiezlikowskiestimatingICONIP} -- a variant of density network \cite{bishop1994mixture}, which returns the parameters of conditional FA for the incomplete image. It should be emphasized that in contrast to \cite{przewiezlikowskiestimatingICONIP}, our DMFA is also trained on missing data, see supplementary material for details.

As our main baseline, we consider Partial Convolution ({\bf PC}) \cite{liu2018partialinpainting}, where the convolution is masked and renormalized to be conditioned on only the valid pixels. To the best of our knowledge, it is the only publicly available method for adapting CNNs to the case of missing data. Additionally, we consider various imputation strategies, which modify input images, but leave CNN architecture unchanged:
\begin{compactitem}
    \item {\bf ACFlow}: missing features are replaced with imputations produced by Arbitrary Conditional Flow \cite{li2019acflow}. This is a recent generative model, which allows for sampling from conditional distribution of missing values. In coherence with the authors of this method, we utilize ACFlow models trained on complete data (other models considered here, including \our{}, are trained only on incomplete data).
    \item {\bf zero}: missing values are replaced with zeros.
    \item {\bf mask}: the results of zero imputation is concatenated with a mask vector indicating missing pixels.
    \item {\bf k-NN}: missing pixels are replaced with mean values of those pixels computed from the $k$ nearest training samples (we use $k=5$). Because of the computational complexity of k-NN, we use a random subset of 5000 training samples for finding neighbors.
\end{compactitem}
We also report the performance of CNN trained and tested on complete datasets (no missing values), which is denoted by {\bf GT} (ground-truth). For a comparison, we only considered methods with publicly available codebases suitable for processing of images and thus many methods described in the related work section \cite{yoon2018gain, mattei2019miwae} have not been taken into account.

\paragraph{Image classification}

    For the classification task, we consider three popular datasets: MNIST, SVHN, and CIFAR-10. We use a classifier built from blocks of Convolution-ReLU-Batch-Normalization-Dropout layers \cite{ioffe2015batchnorm}, followed by two fully-connected layers with a ReLU nonlinearity between them. Detailed parameters of classifiers used for each dataset can be found in the supplemental material.

    Classification accuracies reported in Table \ref{tab:results-accuracy} clearly indicate that \our{} obtains better results than competitive methods. While the advantage of \our{} over second-best mask imputation on MNIST is small, the disproportion on SVHN is large. It is evident that recent methods, such as ACFlow and PC, are unstable and frequently give lower scores than \our{} and a simple zero imputation strategy.

    In Table \ref{tab:results-accuracy-full-images} we also report the accuracies achieved by testing the above classifiers on complete images (without missing values). 
    Although most results are fairly similar on the simplest case of MNIST, experiments on more challenging SVHN and CIFAR10 datasets show that the classifier equipped with \our{} shows the strongest ability to adapt to the complete data.

\begin{table}[]
\centering
        \caption{Classification accuracy when training and testing on incomplete images.} \label{tab:results-accuracy}
\begin{tabular}{cccc}
\hline
& MNIST  & SVHN    & CIFAR10  \\   \hline
zero                     & 0.910                        & 0.679                        & 0.709                        \\ 
mask                     & 0.925                        & 0.493                        & 0.563                        \\ 
k-NN                     & 0.875                        & 0.523                        & 0.660                        \\ 
PC                       & 0.920                        & 0.376                        & 0.522                        \\ 
ACFlow                   & 0.908                        & 0.666                        & 0.671                        \\ 
\our{}  & {\bf 0.931} & {\bf 0.757} & {\bf 0.722} \\ 
GT & 0.990   & 0.889   & 0.811                        \\ 
\hline
\end{tabular}
\end{table}


    \begin{table}[]
    \centering
        \caption{Classification accuracy when training on incomplete images and testing on complete images.}
        \label{tab:results-accuracy-full-images}
\begin{tabular}{cccc}
\hline
  & MNIST  & SVHN   & CIFAR10 \\
  \hline
zero   & 0.984       & 0.843       & 0.761       \\
mask   & 0.986       & 0.827       & 0.752       \\
k-NN   & 0.957       & 0.805       & 0.742       \\
PC     & {\bf 0.987} & 0.759       & 0.731       \\
ACFlow & 0.983       & 0.845       & 0.746       \\
\our{} & 0.986       & {\bf 0.860} & {\bf 0.769} \\
GT & 0.990 & 0.889 & 0.811                      \\
\hline
\end{tabular}
\end{table}
    

\begin{table}[]
\centering
        \caption{Frechet Inception Distance (FID) between true images and images generated by WAE (lower is better).}
        \label{tab:results-fid}
\begin{tabular}{cccc}
\hline
 & MNIST & SVHN   & CelebA \\
 \hline
zero   & 9.72       & 21.97       & 56.62       \\
mask   & 10.96      & 24.36       & 59.45       \\
k-NN   & 9.46       & 19.90       & 61.48       \\
PC     & 9.75       & 22.61       & 64.92       \\
ACFlow & {\bf 9.08} & 21.49       & 59.04       \\
\our{} & 9.60       & {\bf 19.53} & {\bf 56.60} \\
GT & 6.45 & 19.65 & 51.75  \\
\hline
\end{tabular}
\end{table}


    \begin{figure*}[]
        \centering
        \begin{subfigure}{.4\textwidth}
          \centering
          \includegraphics[width=0.95\linewidth]{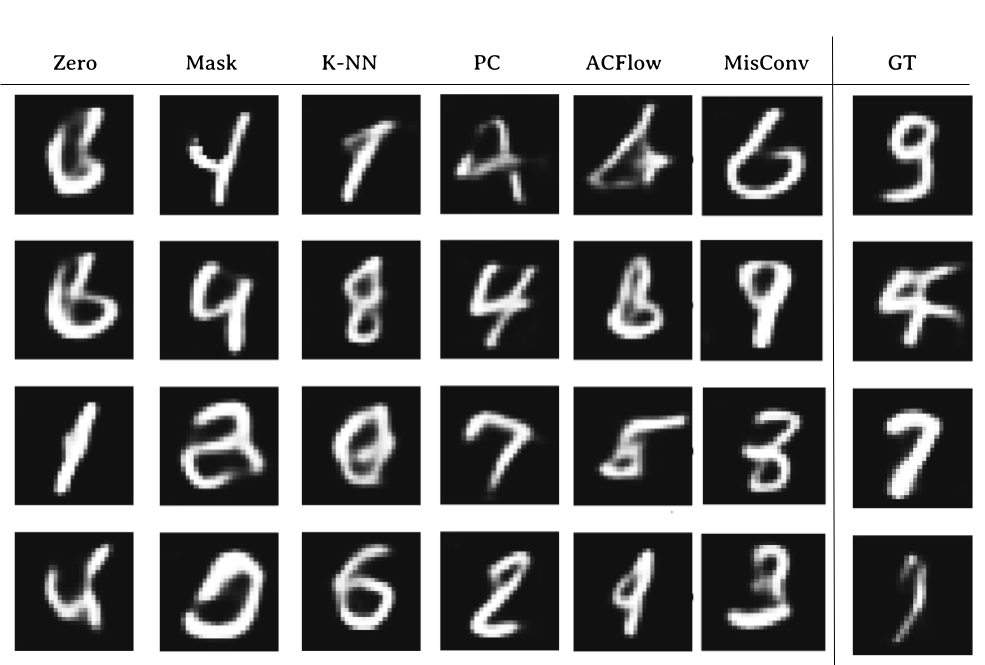}
          \label{fig:gen-mnist}
        \end{subfigure}%
        \begin{subfigure}{.5\textwidth}
          \centering
          \includegraphics[width=0.95\linewidth]{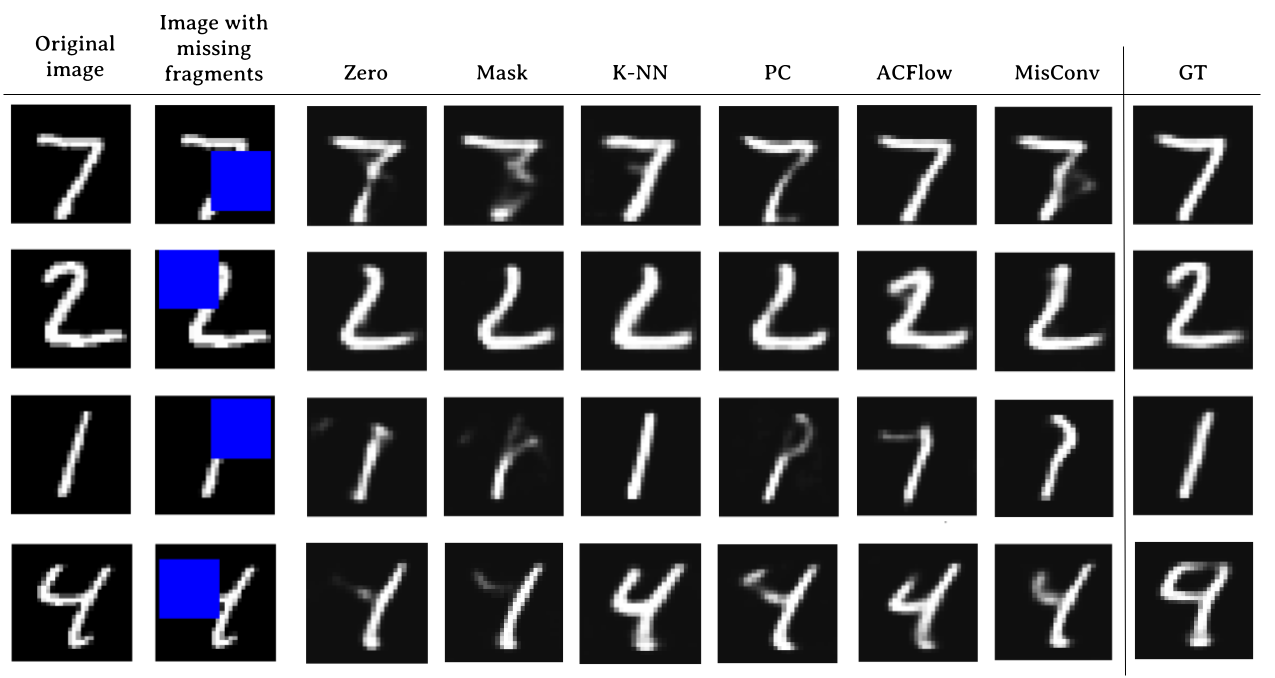}
          \label{fig:recon-mnist}
        \end{subfigure}\\
        \begin{subfigure}{.4\textwidth}
          \centering
          \includegraphics[width=0.95\linewidth]{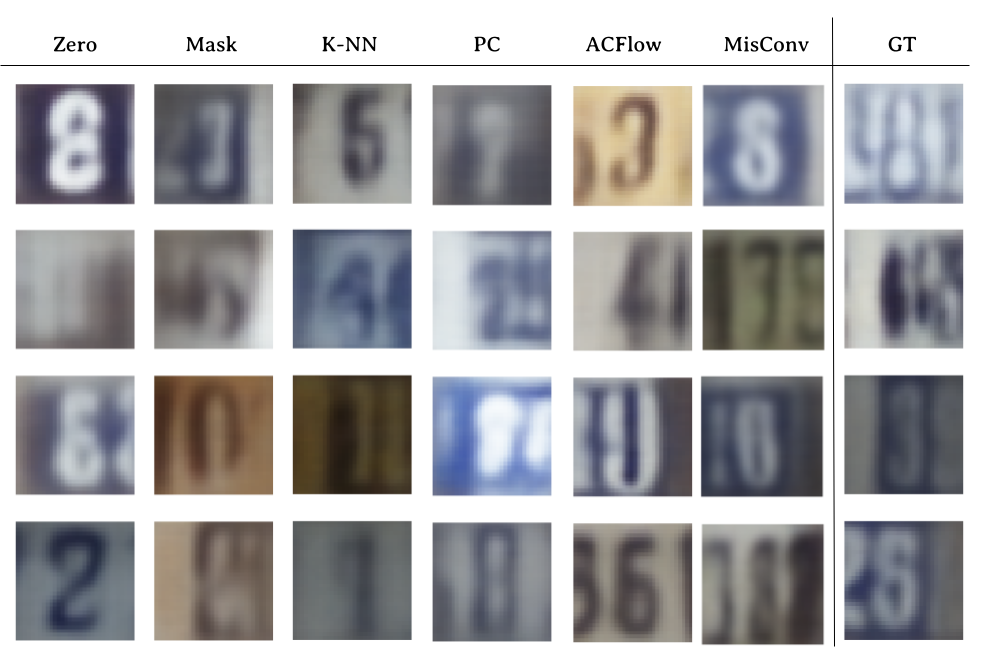}
          \label{fig:gen-svhn}
        \end{subfigure}%
        \begin{subfigure}{.5\textwidth}
          \centering
          \includegraphics[width=0.95\linewidth]{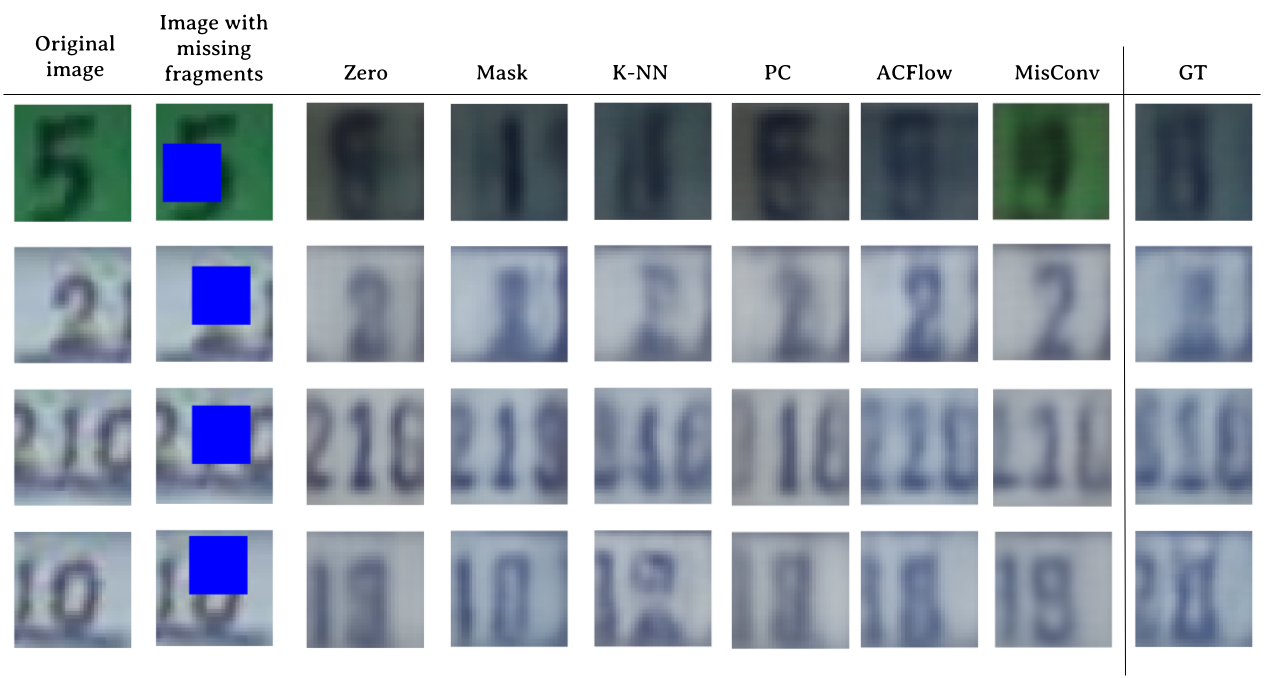}
          \label{fig:recon-svhn}
        \end{subfigure}\\
        \begin{subfigure}{.4\textwidth}
          \centering
          \includegraphics[width=0.95\linewidth]{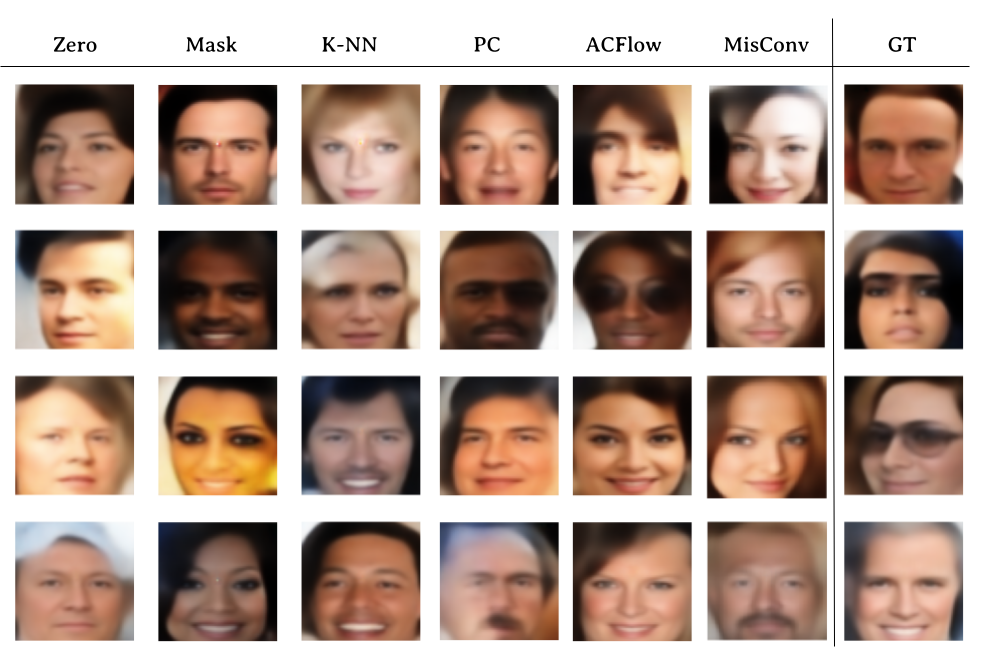}
          \label{fig:gen-celeba}
        \end{subfigure}%
        \begin{subfigure}{.5\textwidth}
          \centering
          \includegraphics[width=0.95\linewidth]{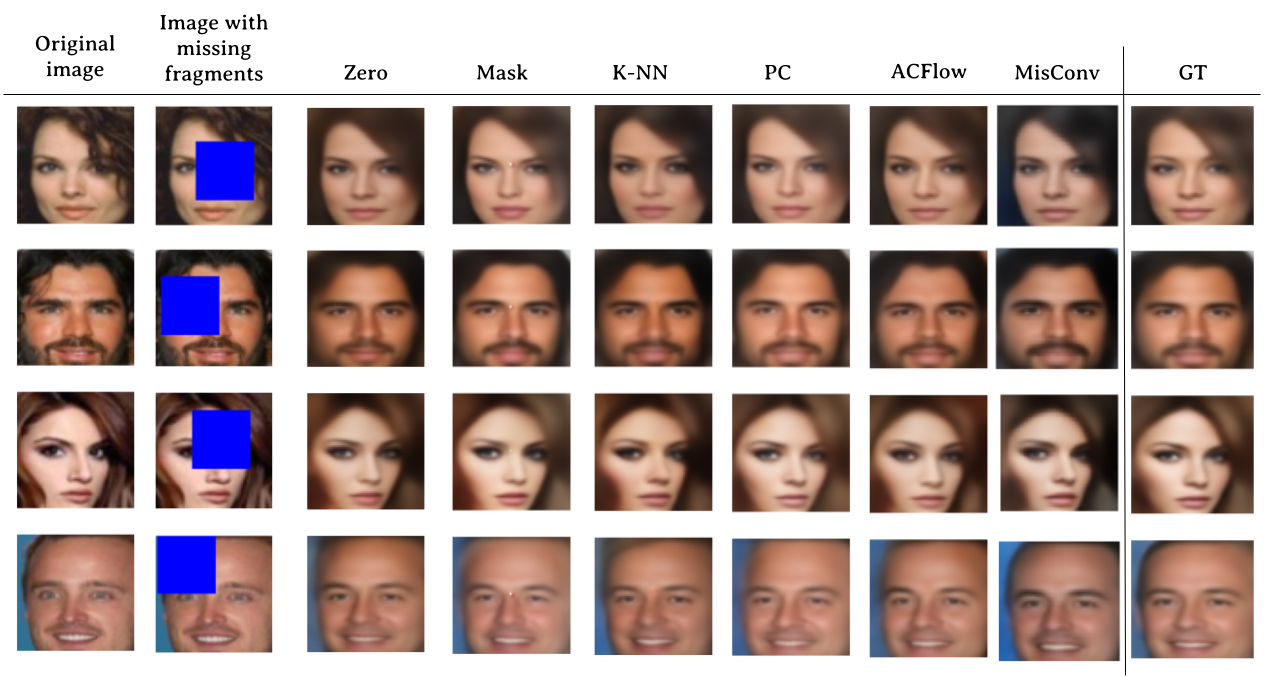}
          \label{fig:recon-celeba}
        \end{subfigure}%
        \caption{Samples (right) and reconstructions (left) produced by WAE models.}
        \label{fig:wae}
    \end{figure*}

    \begin{table}[]
    \centering
        \caption{Structural similarities (SSIM) between true and reconstructed images (higher is better).}
        \label{tab:results-ssim}
\begin{tabular}{cccc}
\hline
     & MNIST  & SVHN   & CelebA \\ 
     \hline
zero   & 0.696       & 0.777       & 0.795       \\ 
mask   & 0.658       & 0.775       & 0.802       \\  
k-NN   & 0.770       & 0.771       & 0.798       \\ 
PC     & 0.662       & 0.774       & 0.786       \\  
ACFlow & 0.628       & 0.787       & 0.805       \\  
\our{} & {\bf 0.803} & {\bf 0.792} & {\bf 0.812} \\ 
GT & 0.944 & 0.811 & 0.832                      \\ 
\hline
\end{tabular}
\end{table}

    
    \begin{table}[]
    \centering
        \caption{Peak signal-noise ratio (PSNR)  between true and reconstructed images (higher is better).}
        \label{tab:results-psnr}
\begin{tabular}{cccc}
\hline
 & MNIST  & SVHN  & CelebA \\ 
 \hline
zero   & 15.05       & 23.81       & 22.74       \\
mask   & 14.26       & 23.64       & 22.94       \\
k-NN   & 16.01       & 23.62       & 22.74       \\
PC     & 14.15       & 23.74       & 22.64       \\
ACFlow & 14.91       & 24.09       & 23.14       \\
\our{} & {\bf 16.27} & {\bf 24.25} & {\bf 23.40} \\
GT & 21.48 & 24.70 & 24.02                      \\ \hline
\end{tabular}
\end{table}
    


\paragraph{Image generation}
    In the second experiment, we use Wasserstein Auto-Encoder (WAE) \cite{WAE} for the image generation task. During training, the WAE reconstruction loss function is evaluated only on the observed pixels. We consider three datasets: MNIST, SVHN, and CelebA. In the case of CelebA dataset, we employ a common preprocessing step -- cropping out the background part of the image and resizing the remaining part to the size of $64 \times 64$. The models were assessed using Frechet Inception Distance (FID) \cite{heusel2018fid, Seitzer2020FID} calculated between the set of ground-truth images and images sampled from WAE models. Due to the large volatility of FID scores, we decided to report the best (lowest) value across the whole training process.
    
    The architecture of the WAE decoder and encoder is based on the DCGAN network \cite{radford2015unsupervised} -- namely, the encoder and decoder are built from repeated blocks of Convolution-ReLU-Batch-Normalization layers, which gradually downsample and then upsample the feature maps. We use fully connected layers to map the encoded feature maps into the latent dimension and a fully connected network to discriminate between random samples and features of real images in the latent space. Details of the architectures can be found in the supplemental material. 
    

    As shown in Table \ref{tab:results-fid}, \our{} outperforms competitive methods on two more challenging datasets, SVHN and CelebA. In the case of SVHN, the FID score obtained by \our{} is even slightly better than the one of the original model trained on complete images. It means that \our{} is capable of producing images of similar quality to standard WAE. While ACFlow performs very well on MNIST, its performance drops significantly on other datasets. PC, our main baseline, again gives unsatisfactory results. For the illustration, we present sample images generated by WAE models, see Figure \ref{fig:wae} (left).
    


\paragraph{Image reconstruction}

    To complement the previous experiment, we also benchmark \our{} in the task of image reconstruction. For this purpose, we use the same architecture and datasets as before and report structural similarity (SSIM) and peak signal-noise ratio (PSNR) between the ground-truth and reconstructed images \cite{wang2004ssim}.

    
    
    As can be seen in Tables \ref{tab:results-ssim} and \ref{tab:results-psnr}, \our{} obtained the best scores on all datasets in terms of both metrics. It is difficult to identify the second-best method because their performances vary across datasets. While k-NN performs quite well on MNIST, it is worse than ACFlow on SVHN and CelebA. In Figure \ref{fig:wae} (right), we show sample reconstructions.

\begin{table}[]
\centering
\caption{Classification accuracy of \our{} using SMFA and DMFA.}\label{tab:mfa-ablation}
\begin{tabular}{lcc}
        \hline
            & SMFA & DMFA\\
            \hline
            MNIST & 0.922 & {\bf 0.931}\\
            SVHN & 0.554 & {\bf 0.757}\\
            CIFAR-10 & 0.640 & {\bf 0.722}\\
            \hline
        \end{tabular}
\end{table}

\begin{table}[]
\centering
\caption{Quality of SMFA and DMFA on SVHN dataset.}\label{tab:dmfa-vs-mfa}
        \begin{tabular}{lcccc}
            \hline
             & NLL & MSE & PSNR & SSIM \\ \hline
            SMFA            & -1127.71     & 5.81         & 26.90         & 0.86         \\
            DMFA           & {\bf -1584.08}     & {\bf 5.34}         & {\bf 30.14}         & {\bf 0.94}         \\ \hline \\
        \end{tabular}
\end{table}
  
\section{Analysis of \our{}} \label{sec:an}

    In this section, we experimentally analyze \our{}. First, we examine the influence of MFA representation on the performance of \our{}. Next, we show that calculating the expected value of the convolution layer applied to MFA is essentially better strategy than using imputation taken from analogical MFA model. We restrict our attention to the classification problem.
    
    \paragraph{MFA representation} \our{} relies on applying the expected value of the convolution operator applied to the MFA representation of the incomplete image. To investigate the influence of MFA parametrization, we consider its two variants. In the first one, the MFA is estimated from the whole dataset, and next its conditional distribution is calculated for each incomplete image. This classical approach, termed here {\bf SMFA} (shallow MFA), was recently explored in \cite{richardson2018gans}. In the second variant, conditional MFA is directly estimated for each incomplete image using density network \cite{przewiezlikowskiestimatingICONIP}. The latter option, called {\bf DMFA} (deep MFA), was used in the previous experiments.
    
    As can be seen in Table \ref{tab:mfa-ablation}, \our{} with DMFA is consistently better than with SMFA. To find the reason behind this behavior, we inspect both models regardless on \our{} (we restrict our attention to the SVHN dataset). Table \ref{tab:dmfa-vs-mfa} presents the Negative Log-Likelihood (NLL) and mean squared error (MSE) of the most probable imputation inferred by the DMFA and SMFA, as well as SSIM and PSNR between the original image and the image with missing pixels replaced by the imputation. Higher quality of DMFA results in a better description of missing values, which impacts the performance of \our{}.

    \paragraph{Comparison with MFA imputation} There is a question whether taking the expected value in \our{} is a better strategy than using a single imputation from an analogical MFA model. To verify this aspect, we compare \our{} with typical CNN applied to the imputation produced by the MFA. It is evident from Table \ref{tab:mis-vs-impute} that taking the whole distribution of missing values is more profitable than using only a point estimate. 
    
    We additionally compare the outputs of the initial classifier layer produced by those two variants, see Figure \ref{fig:misconv-filters}. It can be seen that \our{} is able to differentiate between the observed and missing pixels. This leads us to hypothesize that \our{} allows for processing the signal from the observed and missing regions with varying levels of trust, treating the latter with lower confidence.
   
   \begin{table}[t]
\centering
\caption{Comparing \our{} with MFA imputation in terms of classification accuracy.}\label{tab:mis-vs-impute}
\begin{tabular}{lcc}
        \hline
            & MFA imputation & \our{}\\
            \hline
            MNIST & 0.923 & {\bf 0.931}\\
            SVHN & 0.720 & {\bf 0.757}\\
            CIFAR-10 & 0.716 & {\bf 0.722}\\
            \hline 
        \end{tabular}
\end{table}

\begin{figure}
    \centering
    \includegraphics[width=0.48\textwidth]{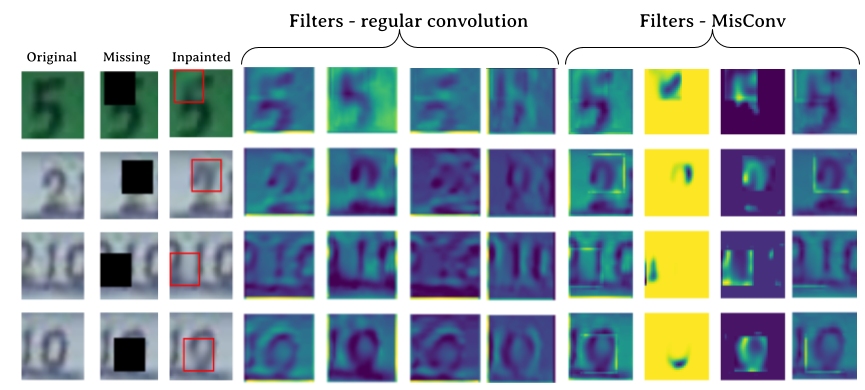}
    \caption{Outputs of the first convolutional layer for MFA imputaton and \our{}. }\label{fig:misconv-filters}
\end{figure}

  
      

 \section{Conclusion}

    We have presented \our{} -- a generalization of CNNs, capable of processing incomplete images. By taking the expected value over all possible replacements after the first convolutional layer, \our{} allows for exploring the uncertainty contained in missing pixels. Making use of MFA as the missing data representation, we were able to efficiently implement the whole framework using matrix operations without the need of sampling. Experimental results confirm that \our{} obtains better results across various tasks and datasets than the recent adaptation of CNNs to the case of missing data as well as typical imputation strategies. 

\section*{Acknowledgement}

The research of M. Przewięźlikowski was funded by the Priority Research Area DigiWorld under
the program Excellence Initiative -– Research University at the Jagiellonian
University in Kraków. The research of M. \'Smieja was funded by the Priority Research Area SciMat under the program Excellence Initiative -- Research University at the Jagiellonian University in Kraków. The research of \L{}. Struski was supported by the Foundation for Polish Science co-financed by the European
Union under the European Regional Development Fund
in the POIR.04.04.00-00-14DE/18-00 project carried out
within the Team-Net programme. The research of J. Tabor was supported by the
National Science Centre (Poland), grant no.  2018/31/B/ST6/00993. For the purpose of Open Access, the authors have applied a CC-BY public copyright licence to any Author Accepted Manuscript (AAM) version arising from this submission.

{
\bibliographystyle{ieee_fullname}
\bibliography{egbib}
}

\clearpage
\appendix

\section{Experimental details}

        In this section, we describe in detail the architectures and hyperparameters of all models used in our experiments. We use the following notation:
        
        \begin{itemize}
            \item  $conv_n$ - a convolution-ReLU-BatchNorm sequence with $n$ convolution filters
            \item $lin_n$ - a fully connected layer with $n$ output features followed by ReLU
            \item $down$ / $up$ - downsampling / upsampling convolution (with stride of 2), respectively:
            \item $drop_p$ - dropout layer with probability $p$ 
        \end{itemize}
        
        For simplicity, we omit the reshaping operations which take place between convolutions and fully connected layers.
    
    \subsection{Mixture of Factor Analyzers}
       \label{app:exp:inpainters}
        
         In this section, we describe the implementation of Deep Mixture of Factor Analyzers (DMFA) and outline the architecture and hyperparameters used for each dataset. For additional information, we refer the reader to \cite{przewiezlikowskiestimatingICML, przewiezlikowskiestimatingICONIP}.
        
        
            \paragraph{MNIST}
                Following \cite{przewiezlikowskiestimatingICONIP}, we use DMFA which uses a small convolutional feature extractor and single fully-connected layers to make final predictions of each MFA parameter $(\mu, A, D)$. The extractor is a sequence of four convolution layers with ReLU activations, with 16, 32, 64 and 32 filters, respectively. The number of predicted factor analyzers is $l=4$. 
                The DMFA is trained for 20 epochs, with a learning rate of $4*10^{-5}$ and a batch size of 48.
                
            \paragraph{SVHN and CIFAR-10}
                Due to the increased dimentionality of those datasets, we use a fully convolutional variant of DMFA, which consists of a fully convolutional feature extractor composed of Convolution-ReLU-BatchNorm blocks \cite{przewiezlikowskiestimatingICONIP}, followed by a downsampling / upsampling convolution with a stride of 2. The network returns three heads that predict $(\mu, A, D)$. Thus the extractor consists of the following layers:
                
                \begin{multline*}
                [conv_{32}] \times 2, down, [conv_{64}]\times 2, down, \\
                [conv_{128}] \times 4, up, [conv_{64}]\times 2, up,[conv_{32}] \times 2
                \end{multline*}
                
                and parameter predictor heads consist of two convolutional layers with 16 filters with a ReLU nonlinearity and a BatchNorm layer between them. 
                
                The number of predicted factor analyzers is $l=4$. We train the DMFA for 100 epochs, with a batch size of 64 and a learning rate of $4*10^{-5}$ for the first 10 epochs and $1*10^{-5}$ afterwards. Similarly to \cite{przewiezlikowskiestimatingICONIP} we note that a fully-convolutional DMFA trained by minimizing only the NLL loss finds it difficult to find a good mean vector $\mu$ of the returned density. We mitigate this by supplying the NLL loss with MSE for the first 10 epochs of training. 
                
            \paragraph{CelebA}
                
                For this dataset we use the same setup as in the case of SVHN / CIFAR-10 images, with two differences: a shorter training time - 50 epochs in total, as well as the size of the feature extractor, which is approximately two times bigger:
                
                \begin{multline*}
                [conv_{32}] \times 4, down, [conv_{64}]\times 3, down,\\ [conv_{128}] \times 8, up, [conv_{64}]\times 4, up, [conv_{32}] \times 4
                \end{multline*}

        \paragraph{Training of DMFA on incomplete data}

            In general, the DMFA model is trained by hiding a portion $x_m$ of a sample $x$ from the model and minimizing the Negative Log-Likelihood of imputation of $x_m$ produced by the model. Since we work with incomplete images, we cannot evaluate the NLL on $x_m$, because we do not have access to complete ground-truth data. Thus we simulate incompleteness of the data by hiding the other part $x_u$ of each data sample $x$. Although the model produces imputations for both $x_u$ and $x_m$, only the NLL of imputation of $x_u \setminus x_m$ is minimized, because this is an artificially created missing region. 
            
    

        
        
        
        
        


        \subsection{Image classifier}
            \label{app:exp:cls}
            For image classification we use a model composed of similar Convolution-ReLU-BatchNorm sequences and downsampling convolutions as in the fully convolutional DMFA, followed by two fully connected layers with a ReLU nonlinearity between them.
            
            \paragraph{MNIST and SVHN}
            
            We use an architecture of: $[conv_{32}, down, [conv_{64}] \times 2, lin_{20}, lin_{10}]$ and train the model with a batch size of 24 and a learning rate of $10^{-3}$ for 10 epochs in case of MNIST dataset and 25 epochs in case of SVHN dataset. 
            
                \paragraph{CIFAR-10}
                
            We use an architecture of: $[conv_{32}, drop_{0.3}] \times 2, down, [conv_{64}, drop_{0.3}] \times 3, lin_{128}, drop_{0.3}, lin_{10}]$ and train the model with a batch size of 64 and a learning rate of $10^{-4}$ for 35 epochs. 
    
        \subsection{Wasserstein Autoencoder (WAE)}
            \label{app:exp:wae}

            In accordance with \cite{WAE}, encoder and decoder sections of the WAE are fully convolutional, with a fully connected layer at the end of the encoder and the beginning of the encoder for transformation into and from the latent space, whereas the discriminator part of the network consists solely of fully connected layers with ReLU between them and a sigmoid layer at the end.
            
            \paragraph{MNIST}
                Architecture:
                \begin{itemize}
                    \item encoder: $conv_{32}, conv_{40}, down, conv_{80}, down, lin_{20}$
                    \item decoder: $lin_{3920}, up, conv_{80}, up, [conv_{40}]\times 2, conv_{1}$
                    \item discriminator: $[lin_{64}] \times 4, lin_{1}$
                \end{itemize}
            
                We train the WAE for 20 epochs with a batch size of 128 and learning rate of $4*10^{-4}$ for encoder and decoder, and $4*10^{-6}$ for the discriminator.  
            
            \paragraph{SVHN}
                Architecture:
                \begin{itemize}
                    \item encoder: $conv_{32}, [conv_{96}]\times 2, down$,\\ $[conv_{192}]\times 2, down, lin_{20}$
                    \item decoder: $lin_{12288}, up, [conv_{192}]\times 2, up$,\\ $[conv_{96}]\times 3,  conv_3$
                    \item discriminator: $[lin_{64}] \times 4, lin_{1}$
                \end{itemize}
                
                We train the WAE for 50 epochs with a batch size of 64 and a learning rate of $4*10^{-4}$ for encoder and decoder, and $4*10^{-6}$ for the discriminator.  
                
            \paragraph{CelebA}
                Architecture:
                \begin{itemize}
                    \item encoder: $conv_{32}, [conv_{64}]\times 4$,\\ $[conv_{128}]\times 4, [conv_{256}]\times 4, [conv{512}]\times 4, lin_{64}$
                    \item decoder: $lin_{8192}, [conv_{512}]\times 4$,\\ $[conv_{256}]\times 4, [conv_{128}]\times 4, [conv_{64}]\times 5, conv_{3}$
                    \item discriminator: $[lin_{512}]\times 4, lin_{1}$
                \end{itemize}
                
                We train the WAE for 50 epochs with a batch size of 16 and a learning rate of $10^{-4}$ for encoder and decoder, and $2.5*10^{-6}$ for the discriminator. 
                   
    \subsection{Implementation details}

        \label{app:implementation}

        The models and experiments have been implemented using the PyTorch framework \cite{PyTorch} and run on  GeForce GTX 1080 and RTX 2080 NVIDIA GPUs. Every experiment described in this work is runnable on a single GPU.

       The code implementing our technique is available publicly at \url{https://github.com/mprzewie/dmfa_inpainting}. The instructions on how to run the code are written in the README file included with the code. 
        
        In addition to our code, in this work we have used the following open-source implementations of baseline imputations: Partial Convolutions\footnote{\url{https://github.com/NVIDIA/partialconv}, available under the BSD3 License}, SMFA\footnote{\url{https://github.com/eitanrich/torch-mfa}}, ACFlow\footnote{\url{https://github.com/lupalab/ACFlow}} and KNN\footnote{\url{https://scikit-learn.org/stable/modules/generated/sklearn.impute.KNNImputer.html}, \url{https://github.com/rapidsai/cuml}, available under the Apache 2.0 License}.

  \section{Other missing data tasks}
    
    
    Throughout the experiments, we used \textbf{square} missing regions which encompass 1/4 of their area. To benchmark the performance of \our{} on more diverse and difficult tasks, we also consider images where half of the image area was occluded with a trapezoid shape (\textbf{trapezoid}), as well as images with 3/4 of pixels randomly removed (\textbf{noise}) -- see Figure \ref{app:fig:mnist-challenging}. Although the missing regions have different shape, we consequently trained the DMFA by simulating additional missing parts with square shape.

        \begin{figure}[]
        \centering
        \includegraphics[width=0.3\textwidth]{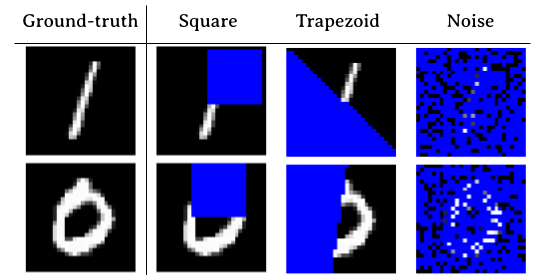}
        \caption{Examples of MNIST images with missing data simulated with square, trapezoid and noise techniques.}
        \label{app:fig:mnist-challenging}
    \end{figure}

    \begin{table}[]
    \centering
        \caption{Classification accuracy when testing on various kinds of incomplete MNIST datasets.}
        \label{app:tab:results-accuracy-challenging}
\begin{tabular}{cccc}
\hline
& square & trapezoid & noise \\ \hline
zero   & 0.91        & 0.891       & 0.946       \\ 
mask   & 0.926       & 0.905       & 0.703       \\ 
k-NN   & 0.874       & 0.854       & 0.912       \\ 
PC     & 0.920       & 0.674       & 0.804       \\ 
ACFlow & 0.908       & 0.799       & 0.788       \\ 
\our{} & {\bf 0.931} & {\bf 0.918} & {\bf 0.966} \\ 
GT & 0.992 & 0.992    & 0.992                     \\ 
\hline
\end{tabular}
\end{table}

    As shown in Table \ref{app:tab:results-accuracy-challenging}, \our{} achieves the best accuracy on all variants of missing data. The above results show promise that \our{} can be leveraged to handle various kinds of missing data problems, even when the percentage of missing data is large.

\section{Computational overhead introduced by using \our{}}

        When computing the expected activation of the convolution based on the MFA representation of the missing data ($\mathbf{\mu, A, d}$), $\mathbf{M}$ must also be applied to  $\mathbf{x}$, $\mathbf{\mu}$, $\mathbf{d}$, as well as each of the $l$ vectors $\mathbf{a_1,...,a_l}$ of which $\mathbf{A}$ is composed. In consequence, there the total of $3+l$ vectors are processed by $\mathbf{M}$, which introduces a computational overhead compared with classical convolutions. Since in practice $l$ is usually small, the effective increase in processing time is insignificant. 
        
        We demonstrate this by measuring the time it takes to perform a forward pass of missing data through a classical convolutional layer and through \our{}. As in all the above experiments, we select $l=4$, which should lead to approximately quadrupled time of a forward pass. We check the time of a forward pass for images of sizes $28\times28$ (as in the MNIST dataset), $32\times32$ (as in the SVHN and CIFAR-10 datasets) and $64\times64$ (as in the CelebA dataset), with batch sizes of 16, 32 and 64. For each such setting, we make 5000 measurements of the time of a forward pass of a batch and report the results in Table \ref{tab:app:time-benchmark} and Figure \ref{fig:app:time-benchmark}. It can be seen that in all cases, the overhead introduced by \our{} indeed increases the processing time approximately by a factor of 4. However, the order of magnitude of this increase is milliseconds. It should also be stressed that this overhead affects only the initial convolutional layer in the entire target neural network, which often consists of tens or even hundreds of layers. 
        \our{} is therefore very efficient in practice, because it introduces a small computational overhead only in the initial layer of the entire neural network.

    \begin{table}[h]
        \centering
        {\footnotesize 

\begin{tabular}{ccclll}
\hline
\multicolumn{1}{l}{}                                                       & \multicolumn{1}{l}{} & \multicolumn{1}{l}{} &                                         & \multicolumn{2}{c}{\textbf{Layer type}}                               \\ \hline
\textbf{Datasets}                                                          & \textbf{I}  & \textbf{M}   & \multicolumn{1}{c}{\textbf{B}} & \multicolumn{1}{c}{Classic convolution} & \multicolumn{1}{c}{MisConv} \\ \hline
\multirow{3}{*}{MNIST}                                                     & \multirow{3}{*}{28}  & \multirow{3}{*}{14}  & 16                                      & 0.00016 ± 0.00006                       & 0.00092 ± 0.00030           \\
                                                                           &                      &                      & 32                                      & 0.00017 ± 0.00006                       & 0.00097 ± 0.00041           \\
                                                                           &                      &                      & 64                                      & 0.00020 ± 0.00008                       & 0.00138 ± 0.00061           \\ \hline
\multirow{3}{*}{\begin{tabular}[c]{@{}c@{}}SVHN, \\ CIFAR-10\end{tabular}} & \multirow{3}{*}{32}  & \multirow{3}{*}{16}  & 16                                      & 0.00022 ± 0.00012                       & 0.00109 ± 0.00039           \\
                                                                           &                      &                      & 32                                      & 0.00037 ± 0.00020                       & 0.00111 ± 0.00046           \\
                                                                           &                      &                      & 64                                      & 0.00034 ± 0.00016                       & 0.00121 ± 0.00047           \\ \hline
\multirow{3}{*}{CelebA}                                                    & \multirow{3}{*}{64}  & \multirow{3}{*}{32}  & 16                                      & 0.00034 ± 0.00017                       & 0.00091 ± 0.00025           \\
                                                                           &                      &                      & 32                                      & 0.00031 ± 0.00012                       & 0.00115 ± 0.00037           \\
                                                                           &                      &                      & 64                                      & 0.00029 ± 0.00008                       & 0.00109 ± 0.00029           \\ \hline
\end{tabular}
}

        \caption{
        Time of a forward-pass through a classical convolutional layer and \our{} layer,
        measured for different image sizes (denoted in the \textbf{I} column) with their respective sizes of missing data squares (denoted in the \textbf{M} column) and different batch sizes  (denoted in the \textbf{B} column).
        }
        \label{tab:app:time-benchmark}
    \end{table}
    
    \begin{figure}
        \centering
        \includegraphics[width=0.5\textwidth]{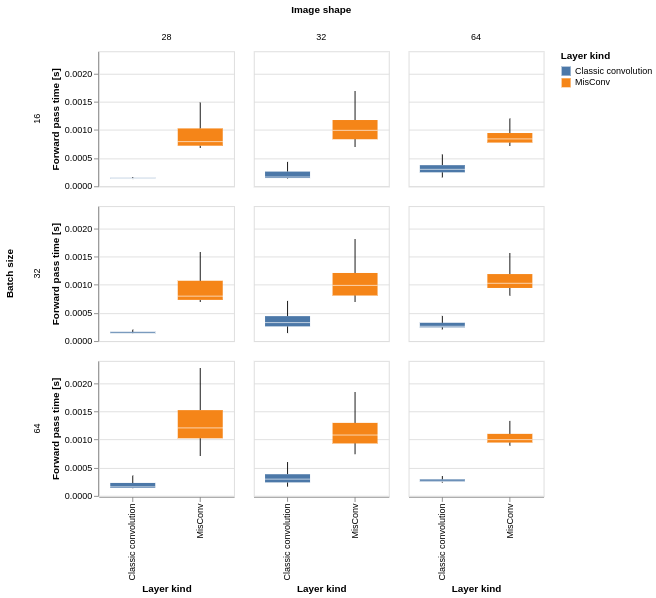}
        \caption{Time of a forward-pass through a classical convolutional layer and \our{} layer,
        measured for different image sizes in a form of boxplots.}
        \label{fig:app:time-benchmark}
    \end{figure}

\end{document}